\documentclass [
	letterpaper,
	10pt,
	conference
] {ieeeconf}
\IEEEoverridecommandlockouts                              
\usepackage{silence}
\usepackage{pifont}
\usepackage[utf8]{inputenc} 
\usepackage[T1]{fontenc}    

\usepackage{amsmath}
\usepackage{amssymb}
\usepackage{amsfonts}
\usepackage{balance}

\usepackage{xurl}
\usepackage{dblfloatfix}

\usepackage{booktabs}
\usepackage{nicefrac}
\usepackage{tikz}
\usetikzlibrary{arrows.meta,positioning,calc}
\definecolor{knobR}{HTML}{0072B2}     
\definecolor{knobVox}{HTML}{CC79A7}   
\definecolor{knobCtl}{HTML}{8A8A8A}   
\usepackage{todonotes}
\usepackage{siunitx}

\usepackage{adjustbox}
\usepackage{multirow}

\usepackage{algorithm}
\usepackage[spaceRequire=true]{algpseudocodex}
\usepackage{xpatch}
\makeatletter
\xpatchcmd{\algorithmic}{\itemsep\z@}{\itemsep=1pt}{}{}
\makeatother

\usepackage{graphicx}       
\graphicspath{{imgs/}}

\makeatletter
\def\endfigure{\end@float}
\def\endtable{\end@float}
\makeatother

\usepackage{amsmath, amsfonts, amssymb, amsthm, bm, mathtools}

\theoremstyle{plain}
\newtheorem{thm}{Theorem}[section]

\newtheorem{proposition}[thm]{Proposition}

\theoremstyle{definition}

\newtheorem*{ex}{Example}

\theoremstyle{remark}

\DeclareMathOperator{\tr}{tr}
\DeclareMathOperator{\Cov}{Cov}
\DeclareMathOperator{\med}{med}

\def\R{{\mathbb{R}}}  

\makeatletter\let\NAT@parse\undefined\makeatother

\usepackage[hidelinks, breaklinks=true]{hyperref}       
\usepackage[caption=false,font=footnotesize]{subfig}
\usepackage{cleveref}
\crefname{figure}{Fig.}{Figs.}
\Crefname{figure}{Fig.}{Figs.}
\crefname{table}{Table}{Tables}
\Crefname{table}{Table}{Tables}
\crefname{section}{Sec.}{Secs.}
\crefname{thm}{Prop.}{propositions}
\Crefname{thm}{Proposition}{Propositions}
\crefname{section}{Sect.}{sections}
\Crefname{section}{Section}{Sections}
\crefformat{equation}{(#2#1#3)}

\usepackage{microtype}

\newlength{\topfloatslack}
\newcommand{\topfloatpad}{\vspace*{\topfloatslack}}
\newcommand{\topfloatunpad}{\vspace{-\topfloatslack}}

\newcommand{\creditarch}{%
	\begin{tikzpicture}[
			font=\footnotesize,
			>={Stealth[length=4.5pt,width=3.2pt]},
			blk/.style={draw=black!55, rounded corners=1.5pt, align=center,
					inner xsep=3pt, inner ysep=3pt, minimum height=7mm, minimum width=17mm},
			io/.style={align=center, inner xsep=1pt},
			flow/.style={->, draw=black!70, line width=0.7pt},
			knobar/.style={->, line width=0.7pt},
			tag/.style={font=\scriptsize, inner sep=1pt},
			lbl/.style={font=\scriptsize, inner sep=1.5pt},
		]
		\node[io] (imu) at (0.3,1.35) {IMU};
		\node[blk] (prop) at (2.1,1.35) {propagate};
		\node[io] (scan) at (0.3,0) {scan};
		\node[blk] (down) at (2.1,0) {voxel\\downsample};
		\node[blk] (assoc) at (5.0,0) {plane match\\$n_i$};
		\node[blk, minimum height=15mm, minimum width=16mm] (upd) at (7.9,0.675)
		{Kalman\\update};
		\node[io] (out) at (10.15,0.675) {$\Sigma_{tt}^+$\\reported};

		\draw[flow] (imu) -- (prop);
		\draw[flow] (prop) -- node[lbl, above] {$\Sigma_{tt}^-$} (upd.west|-prop);
		\draw[flow] (scan) -- (down);
		\draw[flow] (down) -- node[lbl, above] {{\color{knobVox}$N$} points} (assoc);
		\draw[flow] (assoc) -- (upd.west|-assoc);
		\draw[flow] (upd) -- (out);

		\node[tag, text=knobVox, below=3.5mm of down] (kvox) {\textbf{voxel size}};
		\draw[knobar, draw=knobVox] (kvox) -- (down);
		\node[tag, text=knobR, above=3.5mm of upd] (kr) {\textbf{$R$}};
		\draw[knobar, draw=knobR] (kr) -- (upd);

		\node[tag, text=knobCtl, above=3.5mm of prop] (kq) {\textbf{$Q$}};
		\draw[knobar, draw=knobCtl] (kq) -- (prop);
		\coordinate (lb) at ($(assoc.south)+(0,-0.75)$);
		\draw[draw=knobCtl, line width=0.7pt] (upd.south) |- (lb);
		\draw[knobar, draw=knobCtl] (lb) -- (assoc.south);
		\node[tag, text=knobCtl, fill=white] at ($(lb)!0.5!(lb-|upd)$)
		{\textbf{iterations}};

		\node[tag, anchor=west] (eqh) at (11.85,1.62) {\textbf{Kalman update}};
		\node[lbl, anchor=north west] (eqk) at ($(eqh.south west)+(0,-0.10)$)
		{$K_t=\Sigma_{tt}^{-}H_t^\top(H_t\Sigma_{tt}^{-}H_t^\top
				+{\color{knobR}R})^{-1}$};
		\node[anchor=north west, inner sep=0pt] (eq) at ($(eqk.south west)+(0,-0.12)$)
		{$\Sigma_{tt}^{+}=(I-K_tH_t)\,\Sigma_{tt}^{-}$};
		\node[anchor=north west, inner sep=0pt] (eqi) at ($(eq.south west)+(0,-0.18)$)
		{$\begin{aligned}
				(\Sigma_{tt}^{+})^{-1} & =(\Sigma_{tt}^{-})^{-1}+H_t^\top{\color{knobR}R}^{-1}H_t \\
				                       & =(\Sigma_{tt}^{-})^{-1}
				+\underbrace{{\color{knobR}\sigma^{-2}}
					\textstyle\sum_{i=1}^{{\color{knobVox}N}}n_in_i^\top}_{\text{credit }\mathcal{J}}
			\end{aligned}$};
		\node[lbl, anchor=north west] (eqd) at ($(eqi.south west)+(0,-0.15)$)
		{${\color{knobR}R}=\sigma^{2}I$,\quad
			$H_t=[\,n_1^\top;\;\dots;\;n_{{\color{knobVox}N}}^\top\,]$,\quad
			$\lVert n_i\rVert=1$};
		\draw[black!25, line width=0.5pt]
		(11.35,1.95) -- ([yshift=-1.5mm] 11.35,0 |- eqd.south);
	\end{tikzpicture}%
}

\newif\ifanon
\ifdefined\ANONYMOUS\anontrue\else\anonfalse\fi
\newcommand{\SQUARE}{\ifanon\else SQUARE\fi}
\newcommand{\FRAMEURL}{\ifanon \href{https://anonymous.4open.science/r/smfeval-88EB}{\nolinkurl{anonymous.4open.science/r/smfeval-88EB}} \else \href{https://github.com/svendbot/smfeval}{\nolinkurl{github.com/svendbot/smfeval}}\fi}

\title {\LARGE \bf You Should Be Properly Scoring Your Odometry}

\ifanon
	\author {Anonymous submission}
\else
	\author{Ola R\o nning$^{1}$, Usama Saqib$^{1}$, and Andrzej Wąsowski$^{1}$%
		\thanks{$^{1}$Authors are with Software Quality Research Group, IT University of Copenhagen, Denmark. {\tt\small \{oroe,usamas,wasowski\}@itu.dk}}%
	}
\fi

\begin {document}

\maketitle

\begin {abstract}
When we evaluate the performance of our odometry, it is common
practice to score the estimated track against a ground truth.
Unfortunately, scoring uses point metrics, such as the root mean
square error, that ignore the covariance matrix which estimators
like filters and smoothers already report. Using the covariance
matters for two reasons. First, the covariance encodes the
estimator's uncertainty, so it tells us whether the estimator
trusts its own output. An overconfident estimator will not
report itself lost. Second, the covariance weights the error in
each direction of the estimate. Without the covariance, an
estimator is unduly penalized for a high error in an uncertain
direction. Instead of point metrics, we should use strictly
proper scoring rules. These rules score the estimate together
with its reported uncertainty. Strictly proper scoring rules
recover the point metrics when no covariance is reported, and
they diagnose covariance inconsistency when covariance is
reported. Using a one-sided pairwise test, we show that two
estimators can expose overconfidence in at least one of them
without a ground truth. Strictly proper scoring rules and our
pairwise test are available in our open-source framework
\texttt{smfeval}. As a case study, we use \texttt{smfeval} to
assess the uncertainty
quality of the translational component of ground-based
LiDAR-inertial odometry. Across four filters we find
overconfidence -- the worst case reports centimeter certainty
with kilometer error. Knowing the filters are overconfident, we
investigate the mechanism. The investigation traces
overconfidence to filters crediting LiDAR measurements with more
new information than they carry.

\end {abstract}

\vspace{1\baselineskip minus 1\baselineskip}
\section {Introduction}
\noindent
Odometry methods, such as filters and smoothers, estimate a
\emph{belief} over vehicle positions. Unfortunately, standard
practice for \emph{evaluating} odometry methods disregards
beliefs. First, each belief is reduced to a single point
estimate of the position, and the position error is quantified
by Euclidean distance of that single point from the ground
truth. Second, the resulting sequence of errors is reduced to a
single score by taking their average. Uncertainty in that
average is typically reported as if the errors were independent.
Both these reductions discard covariance, though of different
kinds. \looseness -1

The Euclidean distance ignores the covariance of beliefs across
the dimensions of a position. That covariance determines the
shape of the volume within which we expect to find the vehicle.
Under the Euclidean distance, the volume is treated as a sphere,
so an error along a direction of high uncertainty is penalized
just as strongly as one in a direction where the model is less
confident (low uncertainty). Meanwhile, consecutive errors are
correlated, as consecutive ground truth and estimated track
positions are correlated. Treating them as independent narrows
the bars on the reported error score. This narrowing can make
statistically indistinguishable methods look distinct and end up
ranked. For a physical system, this may also mean that errors
are wrongly estimated, often underestimated.

\begin{figure}[t]
	\centering

	\includegraphics [
		width = .8 \columnwidth,
		trim = 0mm 2mm 0mm 2mm,
		clip
	] {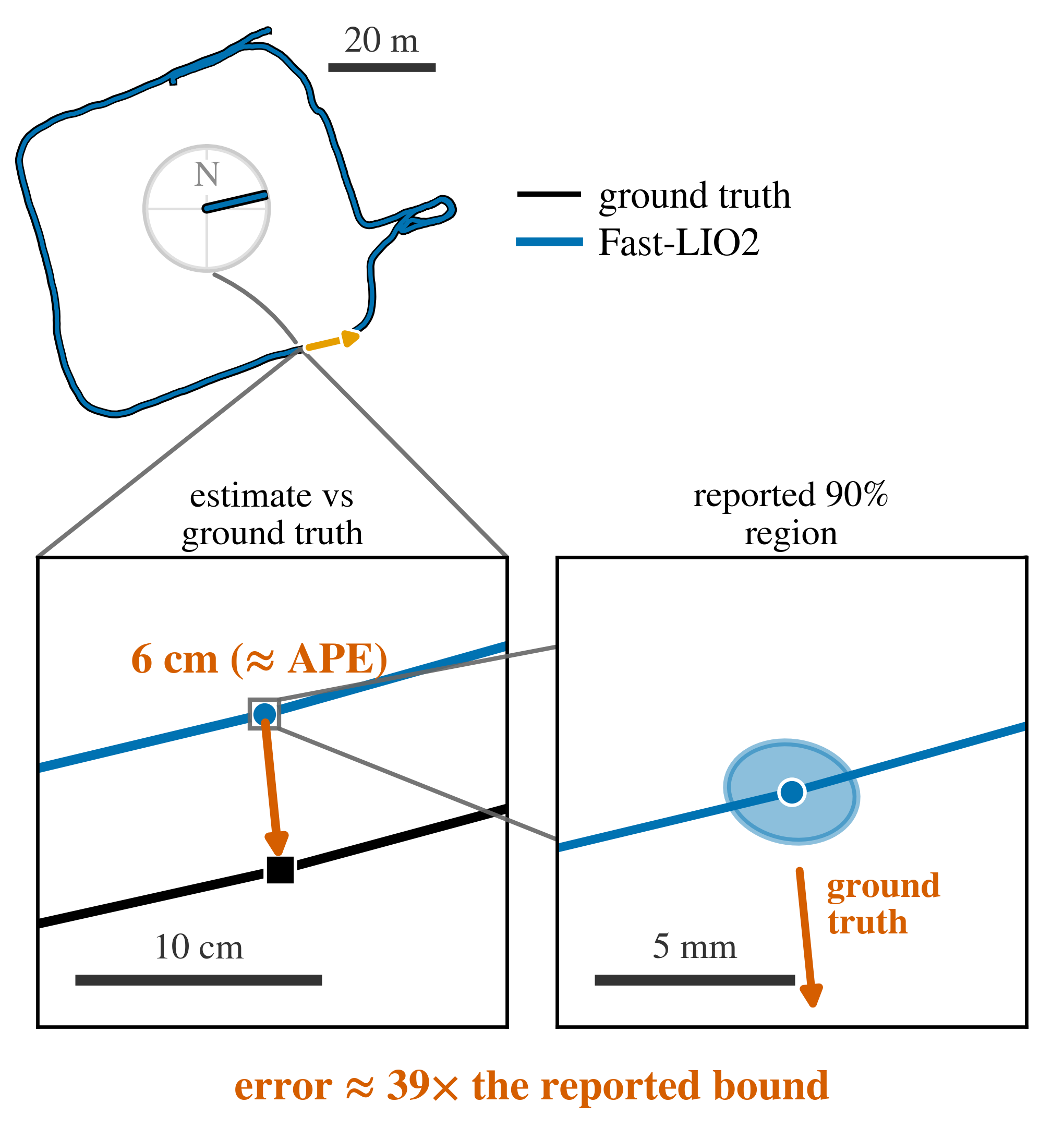}

	\vspace{-1mm plus 1pt minus 2pt}

	\caption{Error against reported certainty for Fast-LIO2 on
	         the Oxford Spires track keble-college-02.\label{fig:overconfidence}}
	\vspace{-1.6ex plus 2pt minus 3pt}

\end{figure}

Gneiting and Raftery \cite{gneiting2007strictly} define
\emph{strictly proper scoring rules} for scoring reported
beliefs against an observed outcome in general. Under these
rules, the expected score is best only if the belief is
consistent with the distribution of outcomes. A single outcome
may score well by chance; only considering many outcomes allows
us to identify a matching belief (a win), and penalize beliefs
too wide or too narrow (a loss). In this work, we apply the
strictly proper scoring rules to odometry evaluation. For track
estimation, the outcome at each time step is the reference
position, so an odometry method is rewarded for assigning
uncertainty that matches the errors it makes. \looseness -1

Crucially, the reference (often designated as ground truth) is
\emph{itself a measurement} and carries \emph{uncertainty of its
own}. Despite this, popular published datasets present it as a
point trajectory, so its uncertainty cannot enter the score. As
a result every odometry method is judged against a reference
that is more certain than it truly is. To emphasize the
uncertainty of the reference, we use the term \emph{gold
standard} for the reference distribution, reserving the term
\emph{ground truth} for single point estimates. \looseness -1

Once the individual positions are scored, the track score is
typically defined as their average. This is not a problem in
itself, but standard tools, such as
\texttt{evo}\,\cite{grupp2017evo}, report too narrow error bars
because they treat scores as independent. Instead, we propose to
use the block bootstrap\,\cite{kunsch1989jackknife}, which
resamples contiguous stretches of the track so that temporal
dependence is preserved. The block length is chosen to capture
the range of dependencies in the sequence of
scores\,\cite{politis2004automatic}. The resulting error bars
are then correctly widened.

We apply the above evaluation methodology to four LiDAR-inertial
odometry filters on ground-based vehicles. Scoring the beliefs
properly shows that the filters are overconfident. In the easy
setting, the filters report millimeter certainty and err by
centimeters (\cref{fig:overconfidence}). In the hard setting,
reported uncertainty widens to centimeters, but every filter
diverges on at least 21 of 55 tracks, and the runs that converge
err by decimeters. The worst case, an urban tunnel, pairs
centimeter certainty with kilometer error (\cref{fig:tunnel}).
Sweeping the configuration parameters traces the overconfidence
to the update step of the filters. The same sweep also reveals
that the covariance is set by the configuration, not by the
residuals. The covariance a planner or fusion node consumes is
therefore decoupled from the surroundings, and a vehicle
depending on these filters cannot tell that it is lost.

The paper makes the following contributions:
\begin{itemize}

	\item A proposal of a methodology for properly scoring
	      odometry track estimations, treating uncertainty as
	      first-class, based on the rules of Gneiting and
	      Raftery \cite{gneiting2007strictly}.

	\item A case study of four LiDAR-inertial odometry filters.
	      We test 67 tracks spanning a feature-rich and a
	      degenerate dataset. The study finds that the reported
	      covariances are $2$ to $7$ orders of magnitude too
	      tight. A parameter sweep over fixed tracks traces the
	      tightness to the configuration, showing that the
	      reported covariance is set predominantly by the
	      parameters rather than by the residuals. \looseness -1

	\item A pairwise overconfidence test (\cref{prop:pairwise})
	      based on the disagreement between two filters. It
	      applies wherever a gold standard is unavailable or
	      itself uncertain, and it is conservative: disagreement
	      beyond what the two covariances allow implies that one
	      filter is overconfident, while agreement is
	      inconclusive.

	\item An open-source framework (\FRAMEURL{}) that implements
	      the evaluation procedure of \cref{sec:case_study},
	      with proper scoring rules, pairwise tests,
	      block-bootstrap intervals, and the \SQUARE{} storage
	      format for estimator beliefs. 

\end{itemize}
We aim to encourage odometry researchers to consider error
beliefs when evaluating methods, raising the assessment
standards for research and precision reporting in this field.
\looseness -1


\section{Related work}
\noindent
Track evaluation is standardized around covariance-free metrics.
The TUM RGB-D benchmark established the convention for metrics
\cite{sturm2012tum}, the \texttt{evo}
tool\,\cite{grupp2017evo} computes them, and
Zhang and Scaramuzza \cite{zhang2018tutorial} codified the track alignment
and aggregation practice. Zhang and
Scaramuzza \cite{zhang2019rethinking} model the
gold-standard track as a continuous-time Gaussian process and
recast pose error as a likelihood, but the estimate remains a
point with no covariance. What we propose instead is to use
proper scoring to measure the accuracy and uncertainty of
odometry jointly. We support validation and ranking of methods.
\looseness -1

The statistic computed by our gold-standard-free test
(\cref{prop:pairwise}) has appeared twice before. Track-to-track
association\,\cite{barshalom1981track} scores the disagreement
between two tracks against their combined covariances to decide
whether the tracks share a target. Similarly,
solution-separation integrity monitoring
\cite{joerger2014solution} compares a position estimate from all
measurements against estimates that each exclude one
measurement, and flags a fault when they disagree. Both works
trust the reported covariances and attribute any disagreements
to other causes. We interpret a disagreement as evidence against
the reported covariances.

Two known effects could explain the overconfidence in our case
study (\cref{sec:case_study}), but do not. First, an extended
Kalman filter linearizes at the estimate rather than the truth,
so each update shrinks the covariance slightly along directions
the measurements do not observe, until after hundreds of updates
the covariance no longer covers the error
\cite{julier2001counter}. The filters we use are already far too
tight after a single update. Second, the closed-form covariance
of point-to-plane registration \cite{censi2007accurate}
understates the error \cite{brossard2020new} because an error
shared by all points of a scan, such as a calibration bias,
averages away under the independent-noise assumption. We test
this directly by integrating such a shared offset out of the
likelihood. Doing so recovers only a small fraction of the gap.


\section{Scoring position belief}\label{sec:scoring_pose_belief}
\noindent
We first recall the classic single point error definition. At
each scan $i \in \{1,\dots,n\}$ the estimator reports a position
$\hat t_i\in\mathbb{R}^3$ and the gold standard supplies $t_i$.
After SE(3) alignment of the two tracks, the position error is
the residual $e_i=\hat t_i-t_i$, and its Euclidean norm is the
absolute pose error (APE). The track score
based on APE is classically defined as the root mean square error
(RMSE) or the mean absolute error (MAE). \looseness -1
\[
	\mathrm{RMSE}=\sqrt{\tfrac1n\textstyle\sum_{i=1}^n\lVert e_i\rVert_2^2},
	\quad
	\mathrm{MAE}=\tfrac1n\textstyle\sum_{i=1}^n\lVert e_i\rVert_2.
\]
Consider an example of a robot traveling along a line with an
  error of \SI{2}{cm} for three consecutive scans, so
  $\mathrm{RMSE}=\mathrm{MAE}=\SI{2}{cm}$. Note that these
  summaries are the same for estimators that report variance of
  \SI{4}{cm^2} for each scan (consistent with the position
  error) and \SI{4}{mm^2} (far from it).

The alternative to APE is the relative pose error. This error
localizes drift within the track by the position error over a
window of $\Delta$ scans,
$e_i^{\Delta}=(\hat t_{i+\Delta}-\hat t_i)-(t_{i+\Delta}-t_i)$
\cite{sturm2012tum}. The rules of this section apply unchanged
for $e_i^{\Delta}$ instead of $e_i$, if the estimator reports a
covariance over $\hat t_{i+\Delta}-\hat t_i$.

\subsection{Proper scoring}

\noindent
Estimators, like filters and smoothers, report a distribution \(
Q \) over the position, known as its belief, for instance, for a
Kalman filter $Q=\mathcal{N}(\hat t,\Sigma)$. A scoring rule
$S(Q,t)$ assigns the estimator belief a score given that the gold
standard reports $t$ for the same scan.  The rule $S(Q,t)$ is
\emph{proper} when no belief scores better, in expectation, than
the reference distribution $t \sim P$ of the gold standard,
\begin{equation}
	\mathbb{E}_{t\sim P}[\,S(P,t)]\le\mathbb{E}_{t\sim P}[\,S(Q,t)]
	\quad\text{for every belief }Q,
	\label{eq:properness}
\end{equation}
and \emph{strictly proper} when equality above forces that
$Q=P$\,\cite{gneiting2007strictly}.

\looseness -1
Note that being proper recognizes that the gold standard carries
uncertainty itself (otherwise the requirement is vacuous).  But
being proper is not enough. Consider the scoring rule assigning
the score $S(Q,t)=\lVert\hat t-t\rVert^2$ to position $t$ with
$\hat t$ the belief mean (this is the non-aggregated core of
RMSE). The rule is proper because it is minimized by $\hat
	t=\mu_P$ (cf.\,\cref{eq:properness}). But this is true for
\emph{any} belief with the correct mean. A \emph{strictly proper}
rule enforces that the truthful belief is the \emph{unique}
minimizer, so any distortion, an inconsistent covariance
included, strictly raises the expected score. We advocate
strictly proper rules. (Hereafter we write \emph{proper} meaning
\emph{strictly proper}.)

\subsubsection*{Distance-sensitive proper rules}
A distance-sensitive rule scores the belief by how far its
probability mass lies from the gold standard. A narrow belief
close to the gold standard scores better than one far away.
Consider the energy score as an example,
\[
	\mathrm{ES}_\beta(Q,t)
	=\mathbb{E}_{X\sim Q}[\,\lVert X-t\rVert_2^\beta]
	-\tfrac12\,\mathbb{E}_{X,X'\overset{\text{iid}}{\sim}Q}[\,\lVert
			X-X'\rVert_2^\beta],
\]
which is strictly proper for $\beta \in (0,2)$ whenever $Q$ has
a finite mean \cite{gneiting2007strictly}. We write
$\mathrm{ES}$, if $\beta=1$. The second term of the energy score
discounts the spread the belief reports, so hedging with an
inflated covariance is penalized. The energy score recovers the
APE if $\beta=1$ and $\Cov(Q)$ approaches
zero\,\cite{gneiting2007strictly}. For our example of a robot
moving along a line, the honest reported belief variance
(\SI{4}{cm^2}) is scored \SI{1.20}{cm} per scan and wins with
the overconfident variance of \SI{4}{mm^2} that receives a score
of \SI{1.89}{cm}. But as covariance $\Cov(Q)$ approaches zero,
the score concentrates near APE, making tight covariances hard
to rank with $\mathrm{ES}$. Instead, we need a scoring that
diverges as the covariance shrinks, so that tight beliefs remain
distinguishable.

\subsubsection*{Density-based proper rules}

A density-based rule instead consults the density the belief
assigns to the gold standard itself. Essentially the only
density-based rule \cite{bernardo1979expected} is the
logarithmic score. For a belief $Q$ with density $q$, we have
\[
	\mathrm{LogS}(Q,t)=-\log q(t).
\]
The score is measured in \emph{nats}. A belief scores one
  \emph{nat} lower than another when it assigns $e$ times the
  density to the gold standard. In our example, the honest
  belief receives $2.1$ nats per scan and the overconfident one
  receives $49$. Where the energy score separated the two
  beliefs by half, the log score separates them by a factor of
  over twenty. Now, understating the covariance drives the
  density at the gold standard toward zero and the score becomes
  arbitrarily large. \looseness -1

\subsubsection*{Log score of a Gaussian belief}
Many Kalman filter variants emit Gaussian beliefs
$\mathcal{N}(\hat t,\Sigma)$, for which the log score splits
into
\begin{equation}
	-\log q(t)
	= \tfrac12\, (\underbrace{\lVert e\rVert^2_{\Sigma}}_{\text{consistency}}
	+ \underbrace{\log\det\Sigma}_{\text{sharpness}})
	+ \tfrac32\log 2\pi.
	\label{eq:logscore}
\end{equation}
The consistency term is the square of the Mahalanobis distance
$\lVert e\rVert_\Sigma=(e^\top\Sigma^{-1}e)^{1/2}$, also called
normalized estimation error squared (NEES), a standard
consistency statistic \cite{bar2001estimation}, and
$e = \hat t -t$ is the difference from the belief mean, like
above. The consistency term penalizes a centimeter of error
along an axis where the belief reports millimeters. When the
error is zero-mean and the reported covariance is the true error
covariance, the scaled error is standard normal and the NEES
follows a $\chi^2_3$ distribution, however sharp the belief is.
The sharpness term measures how concentrated, i.e., peaked, the
belief is \cite{gneiting2007probabilistic}. We call
$-\log\det\Sigma$ the \emph{reported precision}, the log of the
generalized precision $\det\Sigma^{-1}$.

The log score penalizes consistency linearly and sharpness
logarithmically. We see this clearly in our example. The errors
are \SI{2}{cm} and the scalar constant is
$\tfrac12\log2\pi\approx0.9$. The honest \SI{4}{cm^2} belief
scores $0.5$ nats of consistency and $0.7$ of sharpness,
recovering the $2.1$ per scan (see above). The overconfident
\SI{4}{mm^2} belief drops sharpness to $-1.6$ nats, but
increases consistency by $50$. So the hundredfold decrease in
reported variance is penalized heavily because of the asymmetry
between sharpness and consistency. \looseness -1

Consistency and sharpness depend on different degrees of freedom
of $\Sigma$. Write $\Sigma=U\Lambda U^\top$ with $U$ orthogonal
and $\Lambda=\mathrm{diag}(\lambda_1,\lambda_2,\lambda_3)$, so
that the $1$-$\sigma$ ellipsoid of $\mathcal{N}(\hat t,\Sigma)$
has axes along the columns of $U$ with lengths
$\sqrt{\lambda_i}$. We call $(\det\Sigma)^{1/3}$ the scale of
$\Sigma$, and the axis directions $U$ together with the
normalized lengths $\Lambda/(\det\Sigma)^{1/3}$ its shape.
Rescaling $\Sigma\mapsto c\Sigma$ multiplies the scale by $c$
and leaves the shape fixed. The following proposition makes the
asymmetry exact. \looseness -1

\begin{proposition}[Scaling law of the Gaussian log score]\label{prop:scaling}
	Let $\Sigma$ be a $3{\times}3$ covariance,
	$e\in\R^3$ be a fixed error, $c>0$, and $s=3\ln(1/c)$. Then

	\noindent(i) $\lVert e\rVert^2_{c\Sigma}=\tfrac1c\lVert e\rVert^2_\Sigma$ and
	$\log\det(c\Sigma)=\log\det\Sigma-s$.

	\noindent(ii) The log score \cref{eq:logscore} of
	$\mathcal{N}(\hat t,c\Sigma)$ is
	\begin{equation}
		\mathrm{LogS}(s)=\tfrac12\,\lVert e\rVert^2_\Sigma\,e^{s/3}-\tfrac{s}{2}+C,
	\end{equation}
	with $C=\tfrac12\log\det\Sigma+\tfrac32\log2\pi$ collecting
	the terms free of $s$. Its derivative
	${\mathrm{d}\,\mathrm{LogS}/\mathrm{d}s=(\lVert e\rVert^2_{c\Sigma}-3)/6}$
	changes sign once, so $\mathrm{LogS}(s)$ has a unique
	minimum, at $\lVert e\rVert^2_{c\Sigma}=3$.\\ \noindent(iii)
	If $\Sigma_e=\mathbb{E}[ee^\top]$ is positive definite, let
	$A=\Sigma^{-1}\Sigma_e$ and define the shape penalty
	\begin{equation}\label{eq:shape-penalty}
		\psi(\Sigma)=\frac{\tr A}{3\,(\det A)^{1/3}}\ge1,
	\end{equation}
	with equality if and only if $\Sigma$ is a scalar multiple
	of $\Sigma_e$. Then
	$\mathbb{E}\lVert e\rVert^2_\Sigma=\tr A =3\psi(\Sigma)\,(\det\Sigma_e/\det\Sigma)^{1/3}$,
	and $\psi(c\Sigma)=\psi(\Sigma)$ for every $c>0$.
\end{proposition}

\noindent
\textit{Proof.} (i) is $(c\Sigma)^{-1}=\Sigma^{-1}/c$ and
$\det(c\Sigma)=c^3\det\Sigma$. (ii) substitutes (i) into
\cref{eq:logscore} and differentiates in $s$. (iii) is the
inequality between the arithmetic and the geometric mean of the
three eigenvalues of $A$, which are positive, and
$\psi(c\Sigma)=\psi(\Sigma)$ because $\Sigma\mapsto c\Sigma$
sends $A\mapsto A/c$, so that $\tr A\mapsto\tr A/c$ and
$(\det A)^{1/3}\mapsto(\det A)^{1/3}/c$, and the factor $1/c$
cancels in \cref{eq:shape-penalty}.\qed

\medskip

\noindent
Any belief can be made consistent by scaling, but a consistent
belief cannot be made sharper by scaling. By (i), rescaling
$\Sigma$ by $c$ divides the NEES by $c$, so some $c$ brings
$\mathbb{E}\lVert e\rVert^2_\Sigma$ to $3$ whatever the shape of
$\Sigma$. Setting $\mathbb{E}\lVert e\rVert^2_\Sigma=3$ in (iii)
gives $\det\Sigma=\psi^3\det\Sigma_e$, and since any other $c$
breaks consistency, this is the sharpness of the consistent
belief. Consistency thus fixes the scale of $\Sigma$ and says
nothing about its shape, which the NEES does not measure. \looseness -1

\subsubsection*{Inflation}
The \emph{inflation}
\begin{equation}
	k = \frac{\med(\lVert e\rVert^2_\Sigma)}{\med(\chi^2_3)}
	\label{eq:k}
\end{equation}
is the factor by which the reported covariance must be scaled
for the median NEES to reach its calibrated value
$\med(\chi^2_3)\approx2.37$; a consistent filter has $k=1$. The
median makes $k$ insensitive to rare extreme errors that would
dominate a mean. \looseness -1

\subsubsection*{Gold-standard-free overconfidence test}
A second estimator can stand in for the gold standard when
testing for overconfidence. Let estimators $A$ and $B$ report
$(\hat t_A,\Sigma_A)$ and $(\hat t_B,\Sigma_B)$ for the same
scan, with errors $e_A$ and $e_B$, and write
$d=\hat t_A-\hat t_B=e_A-e_B$ for their disagreement. The
\emph{pairwise inflation} is \cref{eq:k} with $d$ in place of
$e$ and $\Sigma_A+\Sigma_B$ in place of $\Sigma$,
\begin{equation}
	k_{\mathrm{pair}} = \frac{\med(\lVert
		d\rVert^2_{\Sigma_A+\Sigma_B})}{\med(\chi^2_3)},
	\label{eq:kpair}
\end{equation}
and measures how far the two reported covariances must be scaled
to match the disagreement.

\begin{proposition}[Pairwise overconfidence]\label{prop:pairwise}
	Let the pair $A,B$ have jointly Gaussian, zero-mean errors
	$e_A,e_B$ with cross-covariance $C_{AB}=\Cov(e_A,e_B)$, and
	write $M=\Sigma_A+\Sigma_B$ and $V=\Cov(e_A)+\Cov(e_B)$. Let
	$a_X$ be the largest eigenvalue of
	$\Sigma_X^{-1/2}\Cov(e_X)\Sigma_X^{-1/2}$ for $X\in\{A,B\}$.
	Then

	\noindent(i) $\Cov(d)=V-(C_{AB}+C_{AB}^\top)$.

	\noindent(ii) If $\Sigma_A=\Cov(e_A)$, $\Sigma_B=\Cov(e_B)$ and
	$C_{AB}=0$, then $\lVert d\rVert^2_M\sim\chi^2_3$ and
	$k_{\mathrm{pair}}=1$.

	\noindent(iii) If $C_{AB}+C_{AB}^\top\succeq0$, then
	$\max(a_A,a_B)\ge k_{\mathrm{pair}}$; in particular, if
	$k_{\mathrm{pair}}>1$ then $\Sigma_A\succeq\Cov(e_A)$ and
	$\Sigma_B\succeq\Cov(e_B)$ cannot both hold.
\end{proposition}

\noindent
\textit{Proof.} (i) follows by expanding $\Cov(e_A-e_B)$. In (ii),
$M=\Cov(d)$ by (i), so $M^{-1/2}d$ is standard normal. For (iii),
write $a=\max(a_A,a_B)$, so that $\Cov(e_X)\preceq a\Sigma_X$ for
both $X$ and $aM\succeq V\succeq\Cov(d)$, the second step by (i)
and the hypothesis on $C_{AB}$. Hence all eigenvalues $\lambda_i$
of $M^{-1/2}\Cov(d)M^{-1/2}$ satisfy $\lambda_i\le a$, and
$\lVert d\rVert^2_M$ is distributed as $\sum_i\lambda_i z_i^2$
with $z_i$ independent standard normal. Since
$\sum_i\lambda_i z_i^2\le a\sum_i z_i^2$ pointwise and
$\sum_i z_i^2\sim\chi^2_3$, the median of $\lVert d\rVert^2_M$
is at most $a\,\med(\chi^2_3)$, so $k_{\mathrm{pair}}\le a$. If
both estimators are calibrated then $a\le1$, giving the
particular case.\qed
\bigskip

\noindent
\Cref{prop:pairwise} makes $k_{\mathrm{pair}}$ a
one-sided test.  A value above one, or an
interval that excludes one, says that at least one estimator
understates its error covariance, but not which one. The bound
assumes the two errors are not negatively correlated; error the
estimators share, from a common map, sensor or calibration,
cancels in $d$ and only makes the test more conservative.
\looseness -1

Returning to our example, let two estimators disagree by
\SI{2}{cm} at every scan, one reporting a variance of
\SI{1}{mm^2} and the other \SI{4}{mm^2}. The pairwise NEES is
$80$, against the $\chi^2_1$ median of ${\sim}0.455$ for scalar
positions. The pairwise inflation is
$k_{\mathrm{pair}}\approx176$, so by \cref{prop:pairwise} at
least one of the two underreports its variance. Had the
estimators reported \SI{100}{mm^2} and \SI{400}{mm^2} instead,
the same disagreement gives $k_{\mathrm{pair}}\approx1.8$, so
the test does not detect overconfidence. This, however, does not
guarantee calibration, since correlated errors shrink $d$. \looseness -1

\begin{figure}[t]
	\centering
	\vspace{1.5mm plus 1pt minus 2pt}

	\includegraphics [
		width = .84 \columnwidth,
		trim = 0mm 1mm 0mm 1mm,
		clip
	] {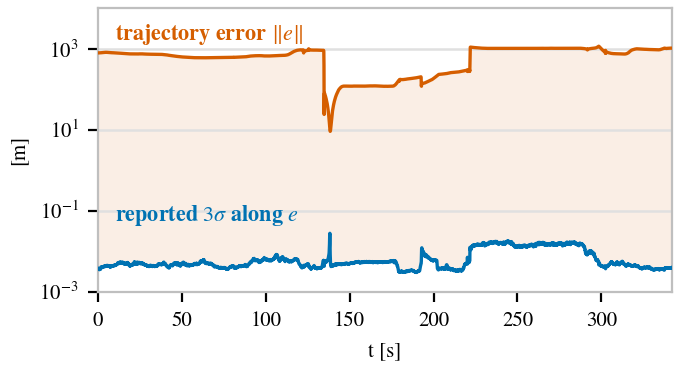}

	\vspace{-1mm plus 1pt minus 2pt}

	\caption{
		Error against reported $3\sigma$ for I2EKF-LO on GEODE
		track.
		\label{fig:tunnel}}%

\end{figure}

\vspace{1 \baselineskip minus 1 \baselineskip}

\subsection{Track summaries}
\noindent
All scores discussed so far apply to a single scan, while
published benchmarks tend to report per-track summaries.  Since
position errors are \emph{autocorrelated}, the error at a scan is
also correlated with the errors at the scans before it; for
example, drift carries over from scan to scan. A track summary is
therefore an estimate whose precision rests on how many of the
$n$ position errors are effectively independent. Treating all
scans as independent overstates the evidence available in the
entire track and tends to inflate the certainty of the
summary. The \emph{block bootstrap} \cite{kunsch1989jackknife}
respects that dependence by resampling contiguous track stretches
rather than single scans \cite{efron1979bootstrap}, with a block
longer than the correlation length of the errors
\cite{politis2004automatic}. The confidence interval is then the
spread of the replicates.
\looseness -1

Let the per-scan score fluctuate with standard deviation
\SI{1}{cm} and lag-one autocorrelation $\rho=0.99$, the
correlation between the errors of consecutive scans, over a
track of $n=1000$ scans. That autocorrelation gives a
correlation length of $(1+\rho)/(1-\rho)\approx200$ scans. With
the scans treated as independent, the $95\%$ interval on the
track mean is $\pm\SI{0.06}{cm}$. The autocorrelation inflates
the variance of the mean by a factor of $179$, so the honest
interval is $\pm\SI{0.83}{cm}$, thirteen times wider. \looseness
-1


\section{Case study: LiDAR-inertial odometry}\label{sec:case_study}
\begin{table*}[tp]
	\topfloatpad
	\centering
	\small
	\caption{Translation accuracy and consistency across
	         filters.}
	\label{tab:consistency}
	\adjustbox{max width=\textwidth}{\setlength{\tabcolsep}{4pt}\begin{tabular}{l r r@{\;}r r@{\;}r r@{\;}r r@{\;}r r}
\toprule
Filter & $n$ & \multicolumn{2}{c}{MAE [m]\,($\downarrow$)} & \multicolumn{2}{c}{Energy [m]\,($\downarrow$)} & \multicolumn{2}{c}{$\mathrm{LogS}$ [nat]\,($\downarrow$)} & \multicolumn{2}{c}{$k$\,(${\to}1$)} & \multicolumn{1}{c}{90\% cov\,(${\to}0.9$)} \\
\cmidrule(lr){3-4}\cmidrule(lr){5-6}\cmidrule(lr){7-8}\cmidrule(lr){9-10}
\midrule
\multicolumn{11}{l}{\textit{Oxford Spires}} \\
Fast-LIO2 & 12 & \textbf{0.088} & $[0.060, 0.197]$ & \textbf{0.087} & $[0.059, 0.196]$ & $1{\times}10^{4}$ & $[4, 93]{\times}10^{3}$ & $5.6{\times}10^{3}$ & $[2.6, 16]{\times}10^{3}$ & 0.000 \\
Faster-LIO & 12 & \textbf{0.117} & $[0.062, 0.508]$ & \textbf{0.116} & $[0.061, 0.507]$ & $2.2{\times}10^{4}$ & $[9.1, 690]{\times}10^{3}$ & $1.4{\times}10^{4}$ & $[3.9, 480]{\times}10^{3}$ & 0.000 \\
Point-LIO & 12 & 0.230 & $[0.119, 0.481]$ & 0.225 & $[0.114, 0.475]$ & $\mathbf{1.5{\times}10^{3}}$ & $[4.8, 86]{\times}10^{2}$ & $\mathbf{5.7{\times}10^{2}}$ & $[1.3, 22]{\times}10^{2}$ & 0.000 \\
I2EKF-LO & 12 & 2.37 & $[0.856, 13.81]$ & 2.37 & $[0.855, 13.81]$ & $2.3{\times}10^{7}$ & $[3, 580]{\times}10^{6}$ & $6.7{\times}10^{6}$ & $[1.3, 240]{\times}10^{6}$ & 0.000 \\
GLIM\,(smoother) & 12 & 0.165 & $[0.089, 0.314]$ & -- &  & $1.1{\times}10^{2}$ & $[1.4, 35]{\times}10^{1}$ & $3{\times}10^{1}$ & $[1.3, 14]{\times}10^{1}$ & 0.043 \\
KISS-ICP & 12 & 0.708 & $[0.312, 1.04]$ & -- &  & -- &  & -- &  & -- \\
\midrule
\multicolumn{11}{l}{\textit{GEODE}} \\
Fast-LIO2 & 55 & 7.96 & $[2.75, 26.05]$ & 7.96 & $[2.74, 26.05]$ & $8.7{\times}10^{6}$ & $[7.6, 290]{\times}10^{5}$ & $2.2{\times}10^{6}$ & $[2.4, 140]{\times}10^{5}$ & 0.000 \\
Faster-LIO & 55 & 1.91 & $[0.708, 24.91]$ & 1.91 & $[0.706, 24.91]$ & $5.7{\times}10^{5}$ & $[1.4, 300]{\times}10^{5}$ & $2.8{\times}10^{5}$ & $[5.5, 2{,}100]{\times}10^{4}$ & 0.000 \\
Point-LIO & 55 & \textbf{0.772} & $[0.407, 31.48]$ & \textbf{0.758} & $[0.396, 31.47]$ & $\mathbf{3.3{\times}10^{3}}$ & $[1, 6{,}100]{\times}10^{3}$ & $\mathbf{1.6{\times}10^{3}}$ & $[7.3, 20{,}000]{\times}10^{2}$ & 0.000 \\
I2EKF-LO & 55 & 31.45 & $[14.05, 42.30]$ & 31.44 & $[14.34, 42.29]$ & $3.5{\times}10^{7}$ & $[1.3, 12]{\times}10^{7}$ & $2.3{\times}10^{7}$ & $[5.8, 44]{\times}10^{6}$ & 0.000 \\
KISS-ICP & 35 & 9.56 & $[0.441, 60.89]$ & -- &  & -- &  & -- &  & -- \\
\bottomrule
\end{tabular}
}
	\topfloatunpad
\end{table*}

\noindent
We implement the rules of \cref{sec:scoring_pose_belief} in
\texttt{smfeval}, an open-source tool that scores an estimator's
belief track against a gold-standard track, given either as
plain poses or with its own covariance. Beliefs are stored in
\ifanon an\else the SQUARE format, an\fi{} extension of the TUM
track format \cite{sturm2012tum} holding a Gaussian, a weighted
particle set, or a plain pose per scan, together with the frame
and perturbation conventions the covariance is expressed in.

Using \texttt{smfeval}, we investigate four filter-based
LiDAR-inertial odometry systems for accuracy and consistency in
one easy and one adversarial setting. Oxford Spires
\cite{tao2024oxford} is the easy case: feature-rich scenes with
distinct geometry to localize against. GEODE
\cite{chen2026heterogeneous} is the adversarial case: its tracks
are chosen to be degenerate, with symmetric geometry or
featureless scenes. A failure on Oxford Spires cannot be
attributed to the environment. On GEODE, accuracy is expected to
degrade, and the question is whether the reported covariance
reveals it. \looseness -1

\subsubsection*{Experimental setup}
\noindent
From Oxford Spires we evaluate 12 outdoor pedestrian tracks with
a LiDAR and an IMU (excluding tracks with known gold-standard
problems) against a gold-standard track obtained by registering
the LiDAR clouds to a terrestrial laser scanner survey map. From
GEODE we evaluate the 55 tracks with a usable gold standard,
spanning seven environments and three LiDAR platforms. We
benchmark four state-of-the-art filters: the LiDAR-inertial
Fast-LIO2 \cite{xu2022fast}, Faster-LIO \cite{bai2022faster},
and Point-LIO \cite{he2023point}, and the LiDAR-only I2EKF-LO
\cite{yu20242}. We use GLIM \cite{koide2024glim}, at its
released defaults, to check whether smoothing over a window of
scans, rather than
marginalizing each scan, removes the inconsistency. GLIM is a
factor-graph smoother and has no per-point $R$ to match, so its
$k$ is comparable as a consistency ratio but its inputs are not
configured like the filters'. KISS-ICP
\cite{vizzo2023kissicp}, LiDAR-only, serves as a covariance-free
accuracy baseline; our runner reads only \texttt{PointCloud2},
so it covers $35$ of the $55$ GEODE tracks, the $20$ absent ones
being the Livox-format gamma platform. Beliefs are recorded at
scan rate; all four filters use a right-perturbation error state
ordered
(translation, rotation). We use the dataset-provided filter
configurations where available, and the filters' released
defaults otherwise. For all filters the measurement noise is
fixed to $R=\SI{e-3}{m^2}$ per LiDAR point (Point-LIO
\SI{e-2}{m^2}). We score the absolute translation error.
Relative pose error cannot be scored because filters do not
report the cross-covariance between poses. We report scores as
if the gold standard were exact because neither dataset ships
its uncertainty. Using a Gaussian-process interpolation
\cite{zhang2019rethinking}, we give it a covariance for the
interpolation between gold-standard poses, though not for the
survey and registration error behind them, and rescore.
This drops NEES by at most \SI{12}{\percent} and changes no
findings.

\subsubsection*{Frames and timing}
\noindent
A scoring rule compares an estimate with a gold-standard
position in the same frame at the same time. The extrinsic
$E\in \mathrm{SE}(3)$ from the filter's frame to the
gold-standard frame multiplies on the right, so the $6{\times}6$
covariance transports by the adjoint $\mathrm{Ad}_{E^{-1}}$.
Writing $J$ for the translation rows of $\mathrm{Ad}_{E^{-1}}$,
\begin{equation}
	\Sigma'_{tt}=J\,\Sigma\,J^\top,\qquad
	J=U^\top\begin{bmatrix}I & -\ell^{\wedge}\end{bmatrix}.
\end{equation}
Here, $U \in \mathrm{SO}(3)$ and $\ell \in \R^3$ are the
rotation and lever arm of $E$. The second block of $J$ carries
orientation uncertainty into position through the lever arm;
rotating the marginal $\Sigma_{tt}$ alone would drop that term.
Each estimated track is aligned to the gold standard by a
whole-track $\mathrm{SE}(3)$ fit; that alignment of world frames
multiplies on the left and leaves $\Sigma$ untouched. From here
on, $\Sigma$ denotes this transported block. For time alignment,
each estimate is paired with the gold standard by
nearest-neighbor timestamp matching within \SI{5}{ms}, under a
centimeter of motion at walking pace. The Gaussian-process
rescoring above also removes this mismatch and does not change
the outcome.

\begin{figure*}[t]
	\topfloatpad
	\centering
	\subfloat[Entry points of the filter's four parameters.
	          Coloring shared with (b).\label{fig:credit_arch}]{%
		\creditarch}\\[-2pt]
	\subfloat[Scaling the shipped parameters. Only $R$ and the
	          voxel size move the credited information.\label{fig:credited_knob}]{%
		\includegraphics[width=0.315\textwidth]{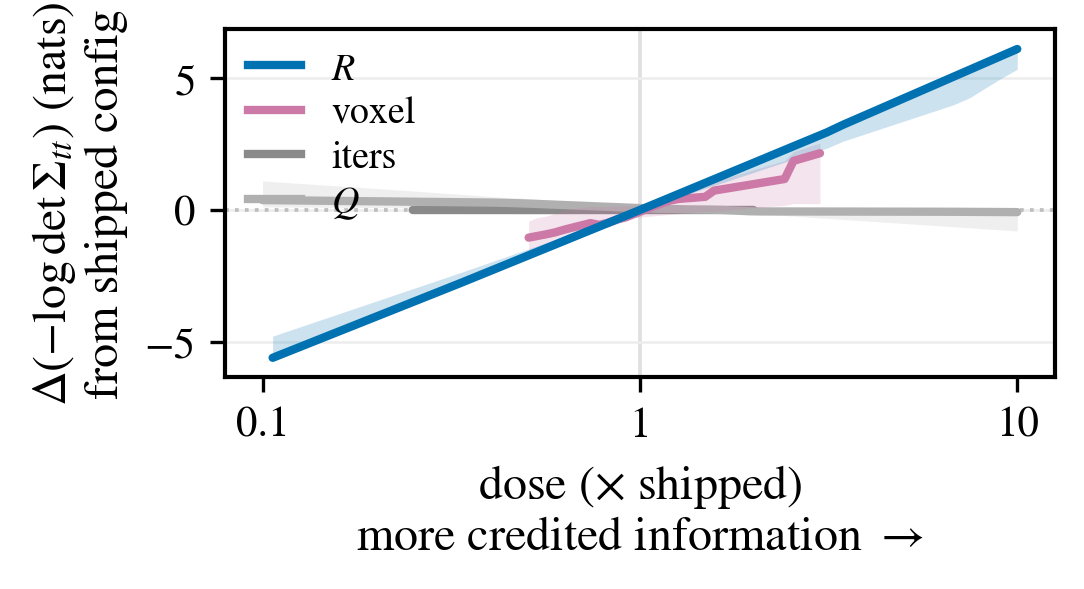}}
	\hfill
	\subfloat[Each nat of precision costs a factor $1.4$ in
	          median NEES.\label{fig:credited_info_nees}]{%
		\includegraphics[width=0.315\textwidth]{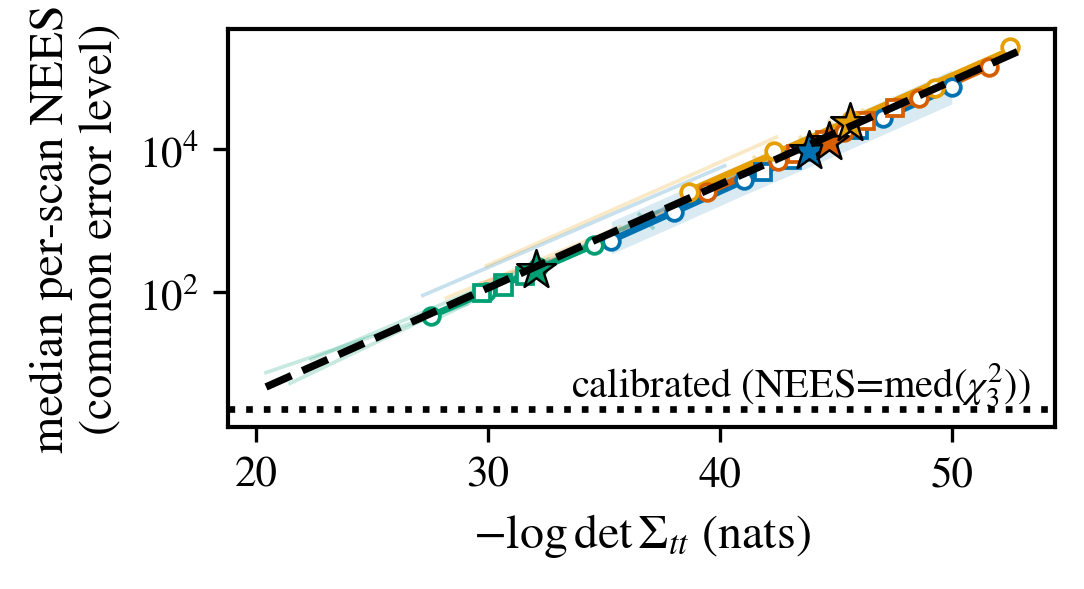}}
	\hfill
	\subfloat[The APE does not follow the dotted line, the error
	          that would justify the reported precision.\label{fig:credited_info_ape}]{%
		\includegraphics [width=0.315\textwidth]{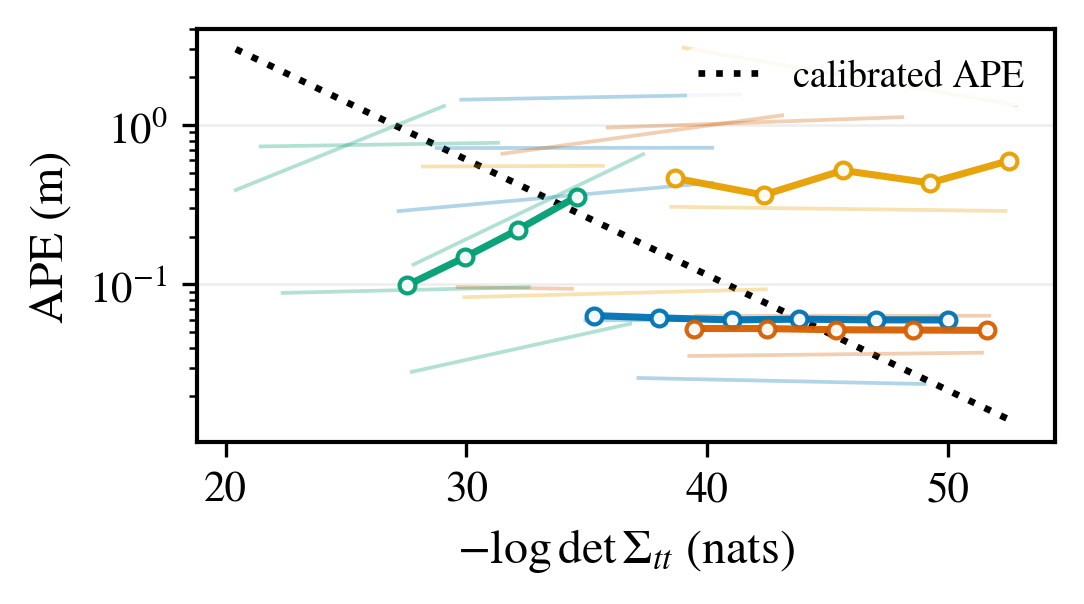}}\\
	\includegraphics [
		width = 0.82\textwidth,
		clip,
		trim = 0mm 3.3mm 0mm 0mm,
	] {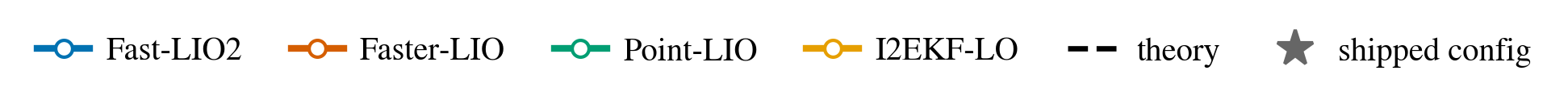}
	\caption{Configuration sweeps of the four filters.}
	         \topfloatunpad
\end{figure*}

\subsection{Every reported covariance is orders of magnitude too tight}
\label{sec:overconfident}
\noindent
\Cref{tab:consistency} reports medians over tracks of per-track
values, with $95\%$ bootstrap intervals from resampling tracks.
Per track, MAE, ES and LogS are means over scans of the absolute
position error, energy score and log score; $k$ is the inflation
of \cref{eq:k} computed per track, and coverage
$\mathrm{cov}_{0.9}$ is the fraction of scans whose gold-standard
position falls inside the reported $90\%$ region. The energy
score collapses to MAE as the covariance shrinks, and here it
tracks MAE to within two percent. That is itself a symptom. The
reported covariances are too small for a distance-sensitive
score to register them. Bold marks the best filter and every
filter it does not significantly beat by a paired Wilcoxon
signed-rank test over the shared tracks, Holm-corrected within
each metric at $\alpha=0.05$. GLIM and KISS-ICP are references
and are excluded from the test. Coverage is reported without
intervals or bold because it is zero on nearly every track, so
intervals collapse and the test does not separate the filters.
\looseness -1

\subsubsection*{Accuracy raises no flags}
\noindent
Point error is well-behaved across the 12 Spires tracks. The
per-filter median MAE (\cref{tab:consistency}) places the three
LiDAR-inertial filters within \SI{0.3}{\meter} of the gold
standard, an accuracy at which nothing would prompt a closer
look. GLIM is comparable at \SI{0.17}{\meter}. The
covariance-free KISS-ICP trails at \SI{0.71}{\meter}, which
bounds what LiDAR alone achieves here; the LiDAR-only I2EKF-LO
falls short of it at \SI{2.4}{\meter}. On GEODE the filters
diverge (MAE $\geq$ \SI{5}{\meter}) on 21 to 43 of the 55 tracks
depending on the filter, as expected from the degenerate
geometry; on these tracks the covariance should grow with the
error.

\subsubsection*{The filters are overconfident}
\noindent
\Cref{tab:consistency} shows that the concentration is
unwarranted for all four filters. Even the best log score runs
to thousands of nats where a consistent filter would score a
few. Splitting the Gaussian log score (\cref{eq:logscore}) into
its two terms, the NEES term contributes $10^{3}$--$10^{7}$ nats
while the sharpness term sits between $-45$ and $-32$ nats, so
the penalty is almost entirely for inconsistency. The inflation
$k$ is between $5.7{\times}10^{2}$ and $6.7{\times}10^{6}$ on
Spires, so the reported covariances would have to be scaled up
by hundreds to millions to match the errors. A calibrated filter
has $\mathrm{cov}_{0.9}=0.9$, an overconfident one less; here
$\mathrm{cov}_{0.9}$ is near zero for every filter, so the
reported $90\%$ ellipsoid almost never contains the truth.
\looseness -1

Divergence does not explain the overconfidence. Between $12$ and
$34$ of the $55$ GEODE tracks converge per filter, with MAE
$<\SI{5}{\meter}$; on those alone, median NEES stays at
$10^{3}$--$10^{5}$ and $\mathrm{cov}_{0.9}$ at zero for every
filter. \looseness -1

Scene degeneracy does not explain it either. We measure
degeneracy by the condition number
$\kappa=\lambda_{\max}/\lambda_{\min}$ of the scan's
point-to-plane information $\sum_i n_i n_i^\top$, with $n_i$ the
surface normal at point $i$ of the downsampled scan, estimated
by principal component analysis. Where the normals concentrate,
$\kappa$ grows, since a translation orthogonal to the normals
leaves every point-to-plane residual unchanged and such poses
cannot be distinguished. Regressing per-scan
$\log\lVert e\rVert^2_\Sigma$ on $\log\kappa$ gives
$R^2\le0.014$ over one order of magnitude of $\kappa$ on Spires
and $6.6$ orders on GEODE, so overconfidence is as severe on
well-conditioned scans as on degenerate ones. The same holds for
the trace, log-determinant and smallest eigenvalue of the
information matrix of the two filters with accessible Jacobians
($R^2\le0.027$). On the same 12 Spires tracks the smoother GLIM,
which keeps a window of scans instead of marginalizing each, is
far less overconfident at a competitive MAE, $k=29.7$, yet still
inconsistent. The severity in the filters therefore points to
what all four share, an iterated error-state Kalman filter
\cite{he2021kalman} with a point-to-plane model against a map
treated as exact.

\subsubsection*{The overconfidence survives without a gold standard}
\noindent
We score the filters against each other. We bring one estimate
into the other's frame by a rigid Umeyama fit on their matched
translations, so the gold standard never enters. Taking the
gauge from the gold-standard alignment instead moves
$k_{\mathrm{pair}}$ by under $5\%$ on Spires and up to $2\times$
on GEODE, small against $k_{\mathrm{pair}}$ itself, so the
result does not depend on the gauge. By \cref{prop:pairwise},
$k_{\mathrm{pair}}$ is a lower bound on the larger of the two
inflations. We observe $k_{\mathrm{pair}}\approx180$--$330$ on
Spires and $15$ to $5.8{\times}10^{3}$ on GEODE, so at least one
filter in every pair is overconfident by two orders of magnitude
on Spires, without any reference to a gold standard.
\Cref{tab:pairwise} reports the median over tracks of the
per-track median with a $95\%$ bootstrap interval over tracks,
GEODE restricted to the tracks on which both filters of the
pair converge. I2EKF-LO is omitted
because it runs two iterated EKFs in place of an IMU and reports
a covariance from each, so there is no single $\Sigma$ to sum
with a partner's. The bound holds under positively correlated
errors; using the gold standard to check this hypothesis, the
symmetric part of the estimated cross-covariance is positive
semidefinite on $9$ of the $10$ Fast-LIO2--Faster-LIO tracks
(two of the twelve are excluded, Faster-LIO diverging) and on
about half the tracks of the pairs involving Point-LIO. Where it
fails, the worst direction inflates the numerator by at most
$1.7\times$, which still leaves $k_{\mathrm{pair}}$ above $100$.
The result survives sweeping the clock offset up to \SI{50}{ms}
and applying the full extrinsic, lever arm included. The $k$
measured against the gold standard sits one to two orders of
magnitude above the lower bound on Spires; the gap is either
error common to the filters or error in the gold standard
itself, and separating the two needs the gold standard's own
covariance, which neither dataset ships.

\begin{table}[bp]
	\centering\small
	\caption{Pairwise inflation per filter pair.}%
	\label{tab:pairwise}
	\vspace{-1mm plus 1pt minus 2pt}
	\setlength \tabcolsep {10pt}
	\begin{tabular}{l r@{\;}r}
		\toprule
		Pair                   & \multicolumn{2}{c}{$k_{\mathrm{pair}}$\,(${\to}1$)}                              \\
		\midrule
		\multicolumn{3}{l}{\textit{Oxford Spires}}                                                                \\
		Fast-LIO2 / Faster-LIO & $3.2{\times}10^{2}$                                 & $[1.4, 4.4]{\times}10^{2}$ \\
		Fast-LIO2 / Point-LIO  & $3.3{\times}10^{2}$                                 & $[7.3, 82]{\times}10^{1}$  \\
		Faster-LIO / Point-LIO & $1.8{\times}10^{2}$                                 & $[5.2, 52]{\times}10^{1}$  \\
		\midrule
		\multicolumn{3}{l}{\textit{GEODE converged}}                                                              \\
		Fast-LIO2 / Faster-LIO & $5.8{\times}10^{3}$                                 & $[4.1, 660]{\times}10^{2}$ \\
		Fast-LIO2 / Point-LIO  & $4{\times}10^{2}$                                   & $[1.9, 348]{\times}10^{1}$ \\
		Faster-LIO / Point-LIO & $1.5{\times}10^{1}$                                 & $[0.81, 5.2]{\times}10^{1}$\\
		\bottomrule
	\end{tabular}
\end{table}

\subsection{The covariance ignores the residuals}
\label{sec:covariance_residuals}
\noindent
The covariances of \cref{tab:consistency} are too tight by
orders of magnitude, and the reason is structural. In the update
step, the information added by a scan of $N$ matched points is
\[
	\mathcal{J}=H_t^\top R^{-1}H_t=\sigma^{-2}\sum_{i=1}^{N} n_i n_i^\top,
\]
with $n_i$ the normals of the matched planes, at most one per
  voxel, and $\sigma$ the configured measurement noise
  (\cref{fig:credited_knob}). We call $\mathcal{J}$ the credit,
  since it is the information the filter credits the scan with.
  The residuals, how far the scan lies off those planes, move the
  mean but do not enter
  $\mathcal{J}$. Since the $n_i$ are unit vectors,
  $\tr\mathcal{J}=N/\sigma^2$, so the total information is fixed
  by the match count and the noise constant before the scan is
  compared to the map. Fast-LIO2 assigns
  $\sigma^2=\SI{e-3}{m^2}$ to each of
  $N\approx5.7{\times}10^{3}$ matched points, which if spread
  evenly gives \SI{0.7}{mm} of standard deviation per axis,
  while the same track is accurate to \SI{6}{cm}
  (\cref{fig:overconfidence}). The two numbers come from
  different sources, and nothing ties them together. This is why
  the overconfidence in \cref{sec:overconfident} does not
  respond to divergence or scene degeneracy. The covariance
  reflects the configuration, and the configuration does not
  change with the scene. \looseness -1

\subsubsection*{Sharpness is paid for in consistency}
\noindent
\Cref{fig:credited_info_nees,fig:credited_info_ape} sweep the
measurement noise $R$ and the voxel size of each filter, which
together set the credited information, across $30$ nats of
precision on $23$ tracks from Spires and GEODE.
\Cref{prop:scaling} gives the sweep two predictions to choose
between. If the sharpening is \emph{earned}, the error shrinks
with the covariance and the NEES holds, so the error falls
$0.072$ decades per nat of precision and the NEES is flat. If
it is \emph{unearned}, the error stays where it was and the
median NEES rises $0.145$ decades per nat. The four filters
follow the $0.145$ line across five orders of magnitude in NEES,
with $22$ of the $23$ filter--scene tracks covering $0.145$ in
their $95\%$ confidence interval. A $0.145$ slope with the
covariance's shape preserved means the error itself did not
move, and \cref{fig:credited_info_ape} shows this directly.
Point-LIO's steeper $0.207$ on keble-college-02 is the
exception, its error growing with the credit because it updates
per point with a gain scaling as $1/R$; dividing out the
measured error returns $0.141$, so its covariance follows the
configuration like the others. The sharpening is unearned. The
credit computed from the normals and $\sigma^2$ alone matches
the precision read off the filter's covariances to within
$10^{-2}$ nats, so the configuration accounts for all of the
reported precision. \looseness -1

The distance from a filter's median NEES down to the consistent
value $2.37$ is $3\ln k$ nats for the inflation $k$ of
\cref{tab:consistency}; on Spires $k$ runs $5.7{\times}10^{2}$
to $6.7{\times}10^{6}$, so the filters sit $19$ to $47$ nats of
precision above consistency. At a median NEES of $10^{4}$ each
further nat costs some $1.7{\times}10^{3}$ nats of log score,
which is why the scores of \cref{tab:consistency} run to
thousands. A filter that let its residuals inform its
covariance, through a noise estimated from the innovations, a
robust kernel, or an observability test, would take less of the
credit as more is offered, and its NEES slope would fall below
$0.145$. None of the four does. \looseness -1

\subsubsection*{The dominant errors displace the whole scan together}
\noindent
The overcounting has a simple source, and the obvious fix does
not remove it. Setting $R=\sigma^2I$, as all four filters do,
declares the $N$ points of a scan to be $N$ independent
measurements, so the credit grows as $N/\sigma^2$. The errors
that dominate a scan are not independent. Extrinsic calibration,
the LiDAR--IMU time offset, motion distortion and drift in the
map displace every point of the scan together, and each is one
unknown rather than $N$. Sampling more points off the same
surfaces cannot resolve a displacement they all share, so the
information a scan carries saturates while the credit keeps
growing. This holds for any likelihood that factorizes over the
points, whatever residual it uses.

The obvious fix is to marginalize the shared displacement as a
per-scan $\delta\sim\mathcal{N}(0,\Sigma_\delta)$, which gives
$R=\sigma^2I+H_t\Sigma_\delta H_t^\top$ and caps the credited
information at
$(\Sigma_\delta+\sigma^2(H_t^\top H_t)^{-1})^{-1}\preceq\Sigma_\delta^{-1}$.
We tested it post hoc on the Fast-LIO2 $R$ sweep of
\cref{fig:credited_info_nees}, where it fails. Applied at a
single update it recovers $3.3$ of the track's $23.7$ nats, a
step in the right direction but a small one. Applied at every
update it does reach nominal, but only at
$\Sigma_\delta^{1/2}=\SI{1.1}{m}$, $277\times$ the \SI{4}{mm}
the innovation actually shows. A displacement that size is not
one the scans have.

Marginalizing a shared displacement is therefore the right
correction applied in the wrong place. What the points share is
not drawn afresh each scan; it accumulates in the map, and the
innovation, which measures a scan against that map, cannot see
it. The uncertainty the filters are missing is the map's, and
carrying it belongs to a different paper. \looseness -1

The likelihood offers two dials for handing back the $23.7$
nats of keble-college-02, which runs on
$N\approx5.7{\times}10^3$ points. Inflating $\sigma$ closes the
gap only at \SI{2}{m} of assumed noise per point, two orders of
magnitude above the sensor's ranging error. Discounting the
count closes it at one or two independent points per scan. Read
this way, $N/k$ is the scan's effective sample size, and the
filters are running at $k$ times it.


\section{Limitations}
\noindent
Five limitations qualify the case study. Scoring first aligns
the estimated track to the gold standard by a rigid transform,
and that alignment is fitted to the same gold standard the
scores then use. Fitting it over the whole track is the
generous choice, and aligning on the first pose alone raises
every median $k$ further. The condition number $\kappa$ is only
a proxy for degeneracy, though for the two filters with
accessible Jacobians every tested scalar summary of their
information matrices explains as little. All four filters build
on the iterated error-state Kalman filter, so the replication
runs across implementations, sensors and datasets rather than
across estimator architectures, with the smoother GLIM the one
outside point. Consecutive errors are strongly dependent, with
a lag-one autocorrelation near $0.99$, so uncertainty attaches
only at the track level. Finally, the inflation $k$ is a median
and $\mathrm{cov}_{0.9}$ is saturated at zero, so the two
together characterize bulk inflation and the presence of
unflagged large errors but not the magnitude of errors in the
tail.

\section{Concluding remarks}
\noindent
Scoring the beliefs of four LiDAR-inertial filters over $67$
tracks finds their reported covariances 2 to 7 orders of
magnitude too tight, and set by the configuration rather than
by the scene. Both findings are reproducible without a gold
standard: the tightness by scoring two filters against each
other, and the dependence on configuration by predicting the
reported precision from the noise constant and the matched
normals alone. Only the claim that the sharpening is unearned
needs the reference. Point error alone would have shown none of
this. Odometry that reports a belief should be scored as one,
and to that end we release \texttt{smfeval},
an open-source implementation of the scores, tests, and
intervals of this paper. \looseness -1

One choice in the methodology deserves more attention than it
usually gets. An odometry filter leaves some degrees of freedom
of its track unpinned, six for a LiDAR-only filter and four
once an IMU fixes gravity, and scoring fits these to the gold
standard before any score is computed. Every score is therefore
conditional on that fit, and the fit uses the reference it is
then judged against. The choice is not neutral. A whole-track
SE(3) fit is the most generous, a first-pose fit the least, and
an inertial filter fitted on six degrees of freedom has its
pitch and roll corrected for free. In this study the gauge
moves $k_{\mathrm{pair}}$ by under $5\%$ on Spires and up to
$2\times$ on GEODE, which is negligible against inflations of
$10^{2}$ to $10^{6}$ but would decide the outcome for a
well-calibrated filter. Two remedies follow. The estimator
should declare its gauge rather than leave the evaluator to
guess it, which \texttt{smfeval}'s format requires; and where a
transform must be fitted, it should enter the belief as an
uncertain quantity rather than a fixed correction, which needs
a gold standard with its own covariance. Neither dataset
supports the second yet.

We score translation only. Proper scores on SO(3) face
intractable normalizers in the natural belief families, and
supplying them is the next step.

\section*{Acknowledgments}

\noindent
\ifanon\else
	This work was supported by the Danish Data Science Academy,
	which is funded by the Novo Nordisk Foundation (NNF21SA0069429).
\fi
We used a generative language model for grammatical
refinement, literature review assistance, and code debugging.

\balance
\bibliography{references}

\end{document}